\documentclass[letterpaper, 10pt, conference]{ieeeconf}

\usepackage{graphicx}
\usepackage{subfig}
\usepackage{amsmath}
\usepackage{amsthm} 
\usepackage{amssymb}
\usepackage[font=footnotesize]{caption} 
\usepackage{subfig} 
\usepackage[noadjust]{cite} 
\usepackage{color}
\usepackage{algorithm} 
\usepackage{algpseudocode} 
\usepackage{enumerate}
\usepackage[hidelinks,colorlinks=false]{hyperref}
\usepackage[left=0.75in, right=0.75in, top=0.75in, bottom=0.75in]{geometry}
\usepackage{authblk} 

\usepackage{bm}
\usepackage{tikz}
\usetikzlibrary{calc} 
\usetikzlibrary{shapes} 
\usetikzlibrary{chains}
\usetikzlibrary{fit}
\usetikzlibrary{arrows}
\usetikzlibrary{decorations.text} 
\usetikzlibrary{decorations.markings}
\usetikzlibrary{decorations.pathmorphing} 
\usetikzlibrary{shadows}
\usetikzlibrary{patterns}
\usetikzlibrary{matrix}
\usepackage{pgfplots}
\usepackage[europeanresistors]{circuitikz}
\usepackage[outline]{contour} 
\contourlength{1.5pt}

\newtheorem{theorem}{Theorem}

\newcommand{\T}{\mathrm{T}}

\newcommand{\R}{\mathbb{R}}

\newcommand{\sk}[1]{\left[#1\right]_\times} 

\graphicspath{{figures/}}

\begin{document}

\title{Time Derivative of Rotation Matrices: A Tutorial}
\author{Shiyu Zhao
\thanks{Shiyu Zhao is with the Department of Automatic Control and Systems Engineering,
    University of Sheffield, UK.
    {\tt\small szhao@sheffield.ac.uk}}
}
\IEEEoverridecommandlockouts 
\date{}
\maketitle
\begin{abstract}
The time derivative of a rotation matrix equals the product of a skew-symmetric matrix and the rotation matrix itself.
This article gives a brief tutorial on the well-known result.
\end{abstract}

\section{Introduction}

The attitude of a ground or aerial robot is often represented by a rotation matrix, whose time derivative is important to characterize the rotational kinematics of the robot.
It is a well-known result that the time derivative of a rotation matrix equals the product of a skew-symmetric matrix and the rotation matrix itself.
One classic method to derive this result is as follows \cite[Sec~4.1]{bookMurray}, \cite[Sec~2.3.1]{MayiBook}, and \cite[Sec~4.2.2]{bookRobotModelingControl} (see \cite{HamanoRotationDerivative} for other methods).
Let $R(t)\in\R^{3\times 3}$ with $t\ge0$ be a rotation matrix satisfying $R(t)R^\T (t)=I$ for all $t$ where $I$ is the identity matrix.
Taking time derivative on both sides of $R(t)R^\T (t)=I$ gives
$$\dot{R}(t)R^\T (t)+R(t)\dot{R}^\T (t)=0,$$
which indicates that $S(t)\triangleq\dot{R}(t)R^\T (t)$ is a skew-symmetric matrix satisfying $S(t)+S^\T (t)=0$ for all $t$, and consequently
$$\dot{R}(t)=S(t)R(t).$$
The above derivation is simple, but it is not straightforward to see the precise physical meaning of $S(t)$ (though $S(t)$ corresponds to an angular velocity vector, it is unclear which reference frame this vector is expressed in).
This article gives another simple derivation, which is essentially a reorganization of the derivation in \cite{bookMurray,MayiBook,bookRobotModelingControl}, to clarify the precise physical meanings of the quantities in the expression of the time derivative of a rotation matrix.

\emph{Notation:}
For any vector $w=[w_1,w_2,w_3]^\T\in\R^3$, define the skew-symmetric operator $\sk{\cdot}$ as
\begin{align}\label{eq_skewSymmetricOperator}
    \sk{w}\triangleq\left[
      \begin{array}{ccc}
        0 & -w_3 & w_2 \\
        w_3 & 0 & -w_1 \\
        -w_2 & w_1 & 0 \\
      \end{array}
    \right]\in\R^{3\times 3}.
\end{align}
The skew-symmetric operator is useful because it can convert a cross product of two vectors into a matrix-vector product. More specifically, for any $w,x\in\R^3$, it can be easily verified that $w\times x=\sk{w}x$.
Another useful property is that for any $w\in\R^3$ and any rotation matrix $R\in\R^{3\times3}$ satisfying $RR^\T =I$ and $\det(R)=1$ it holds that $\sk{Rw}=R\sk{w}R^\T $ \cite[Section~4.2.1]{bookRobotModelingControl}.

\section{Time Derivative of Rotation Matrices}

Consider two reference frames $A$ and $B$ in the three-dimensional space (see Figure~\ref{fig_demoTwoFramesSameOrigin}).
Assume the origins of the two frames collocate with each other.
Suppose frame $A$ is fixed and frame $B$ is rotating.
In the area of robotics, frame $A$ usually corresponds to the \emph{world frame} fixed on the ground, and frame $B$ usually corresponds to the \emph{body frame} attached to the body of a robot.

\begin{figure}
  \centering
  \def\myscale{0.45}
  \def\length{4}
\begin{tikzpicture}[scale=\myscale,>=stealth,every shadow/.style={shadow scale=1.2, shadow xshift=1pt, shadow yshift=-1pt,fill=black!70!white, path fading={circle with fuzzy edge 15 percent}}]

\begin{scope}[shift={(0,0)},rotate=0]
\coordinate (A)  at (0,0);
\coordinate (Ax) at (-\length*0.6,-\length*0.5);
\coordinate (Ay) at (\length,0);
\coordinate (Az) at (0,\length);
\draw[->,thick] (A)--(Ax) node[left] {$x_A$};
\draw[->,thick] (A)--(Ay) node[right] {$y_A$};
\draw[->,thick] (A)--(Az) node[above] {$z_A$};
\end{scope}

\begin{scope}[shift={(0,0)},rotate=-25]
\coordinate (B)  at (0,0);
\coordinate (Bx) at (-\length*0.6,-\length*0.5);
\coordinate (By) at (\length,0);
\coordinate (Bz) at (0,\length);
\draw[->,thick,blue] (B)--(Bx) node[left] {$x_B$};
\draw[->,thick,blue] (B)--(By) node[right] {$y_B$};
\draw[->,thick,blue] (B)--(Bz) node[above] {$z_B$};

\coordinate (w) at (\length*0.6,\length*0.8);
\draw[-angle 90,blue,dashed] (B)--(w) node[right]{$w_B$};
\node [circle,minimum size=8,draw=none] (temp) at ($(w)+(-1,-1)$) {};
\path (temp) edge[loop,out=-30,in=60,looseness=6,-angle 90,blue]  (temp);

\end{scope}

\end{tikzpicture}
  \vspace{-15pt}
  \caption{Frame $A$ is fixed while frame $B$ is rotating.}
  \label{fig_demoTwoFramesSameOrigin}
\end{figure}
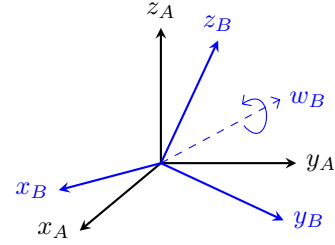

In the sequel, the time variables of all the matrices and vectors are omitted for the sake of simplicity.
Let the rotation matrix $R_B^A\in\R^{3\times 3}$, which satisfies $(R_B^A)^{-1}=(R_B^A)^\T $ and $\det(R_B^A)=1$, represent the rotational transformation from frame $B$ to frame $A$.
For any point in the space, suppose $P_A\in\R^3$ and $P_B\in\R^3$ are its coordinates expressed in frames $A$ and $B$, respectively, then $$P_A=R_B^AP_B.$$
Let $R_A^B=(R_B^A)^\T $ be the rotation from frame $A$ to frame $B$.

Suppose $w_B\in\R^3$ is the angular velocity of frame $B$ (relative to frame $A$) expressed in frame $B$.
The vector $w_B$ quantifies the rotational movement of frame $B$: $\|w_B\|$ equals to the angular rate, which quantifies how fast frame $B$ is rotating, and $w_B/\|w_B\|$ indicates the axis of the rotational movement.
Since the angular velocity is a vector, it can also be expressed in frame $A$ as $w_A\in\R^3$, which satisfies $$w_A=R_B^Aw_B.$$

The following is the main result on the relation between rotational transformations and angular velocities.

\begin{theorem}[\textbf{Time Derivative of Rotation Matrices}]\label{thm_mainResult}
The time derivative of the rotational transformations $R_B^A$ and $R_A^B$ are expressed as
\begin{align}
\label{eq_rotationTimeDerivative1} \dot{R}_B^A&=[w_A]_{\times}R_B^A \\
\label{eq_rotationTimeDerivative2} \dot{R}_B^A&=R_B^A[w_B]_{\times}\\
\label{eq_rotationTimeDerivative3} \dot{R}_A^B&=-R_A^B[w_A]_{\times}\\
\label{eq_rotationTimeDerivative4} \dot{R}_A^B&=-[w_B]_{\times}R_A^B
\end{align}
\end{theorem}
\begin{proof}
We first prove \eqref{eq_rotationTimeDerivative1}.
Consider an arbitrary point fixed in frame $B$.
If $P_A$ and $P_B$ are the coordinates of this point in frames $A$ and $B$, respectively, then $P_B$ is constant since the point is fixed in frame $B$, and $P_A$ is time-varying since frame $B$ is rotating.
As a result, we have $\dot{P}_B=0$.
Taking time derivative on both sides of $P_A=R_B^A P_B$ yields
\begin{align}\label{eq_proofEq1}
\dot{P}_A=\dot{R}_B^A P_B.
\end{align}
On the other hand, by the relation between linear and angular velocities, we have
\begin{align}\label{eq_proofEq2}
\dot{P}_A=w_A\times P_A=[w_A]_{\times}P_A.
\end{align}
Substituting \eqref{eq_proofEq2} into \eqref{eq_proofEq1} gives
\begin{align}\label{eq_proofEq3}
\dot{R}_B^A P_B=[w_A]_{\times}P_A=[w_A]_{\times}R_B^A P_B.
\end{align}
Since $P_B$ may be arbitrarily chosen, equation~\eqref{eq_proofEq3} holds for arbitrary $P_B\in\R^3$ and hence implies \eqref{eq_rotationTimeDerivative1}.

Equations \eqref{eq_rotationTimeDerivative2}--\eqref{eq_rotationTimeDerivative4} can be derived from \eqref{eq_rotationTimeDerivative1}.
In particular, by the property of the skew-symmetric operator, we have
\begin{align}\label{eq_wAwBconvert}
\sk{w_A}=\sk{R_B^Aw_B}=R_B^A\sk{w_B}(R_B^A)^\T .
\end{align}
Substituting \eqref{eq_wAwBconvert} into \eqref{eq_rotationTimeDerivative1} yields $\dot{R}_B^A=\sk{w_A}R_B^A=R_B^A\sk{w_B}(R_B^A)^\T R_B^A=R_B^A\sk{w_B}$, which is \eqref{eq_rotationTimeDerivative2}.
Taking transpose on both sides of \eqref{eq_rotationTimeDerivative1} gives $\dot{R}_A^B=-R_A^B\sk{w_A}$, which is \eqref{eq_rotationTimeDerivative3}.
By substituting \eqref{eq_wAwBconvert} into \eqref{eq_rotationTimeDerivative3} leads to \eqref{eq_rotationTimeDerivative4}.
\end{proof}

As indicated by Theorem~\ref{thm_mainResult}, the expression of the time derivative depends on the definitions of the rotation transformation and angular velocity.
One should be clear about their physical meanings before applying \eqref{eq_rotationTimeDerivative1}--\eqref{eq_rotationTimeDerivative4}.

\section{Practical Consideration in Robotic Motion}

For a robot equipped with an inertial measurement unit (IMU), the value of $w_B$, which is the angular velocity of the robot relative to the world frame expressed in its body frame, can be directly measured.
As a result, equations \eqref{eq_rotationTimeDerivative2} and \eqref{eq_rotationTimeDerivative4}, i.e.,
\begin{align*}
\dot{R}_B^A&=R_B^A[w_B]_{\times}\\
\dot{R}_A^B&=-[w_B]_{\times}R_A^B
\end{align*}
are particularly useful in practice.

It must be noted that the origins of frames $A$ and $B$ are assumed to collocate with each other in Theorem~\ref{thm_mainResult}.
This assumption is, however, usually not satisfied for moving robots because the body frame may translate in the space (see Figure~\ref{fig_demoTwoFrames}).
Nevertheless, \eqref{eq_rotationTimeDerivative2} and \eqref{eq_rotationTimeDerivative4} still holds in this case. To prove that, we may introduce an intermediate frame $A'$ whose axes are parallel to those of frame $A$ and origin collocates with the origin of frame $B$.
By considering frames $A'$ and $B$, we have $\dot{R}_B^{A'}=R_B^{A'}[w_B]_{\times}$ and $\dot{R}_{A'}^B=-[w_B]_{\times}R_{A'}^B$.
Since the axes of frame $A'$ are parallel to those of frame $A$, we always have $R_B^{A'}=R_B^A$ and $R_{A'}^B=R_{A}^B$ and consequently \eqref{eq_rotationTimeDerivative2} and \eqref{eq_rotationTimeDerivative4} still holds (note $w_B$ remain the same).
On the other hand, if the origins of frames $A$ and $B$ do not collocate, equations \eqref{eq_rotationTimeDerivative1} and \eqref{eq_rotationTimeDerivative3} do not hold any more because $w_A'\ne w_A$ due to the nonzero translation between frames $A'$ and $A$.
With the above discussion, we know equations \eqref{eq_rotationTimeDerivative2} and \eqref{eq_rotationTimeDerivative4} are more useful than equations \eqref{eq_rotationTimeDerivative1} and \eqref{eq_rotationTimeDerivative3} in practice.

\begin{figure}
  \centering
  \def\myscale{0.45}
  \includegraphics[width=0.6\linewidth]{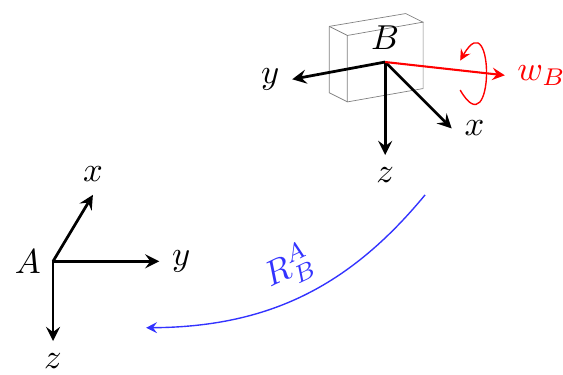}
  \caption{The world frame $A$ and body frame $B$ for a robot moving in the three-dimensional space.}
  \label{fig_demoTwoFrames}
\end{figure}

\begin{figure}
  \centering
  \def\myscale{0.45}
  \includegraphics[width=0.6\linewidth]{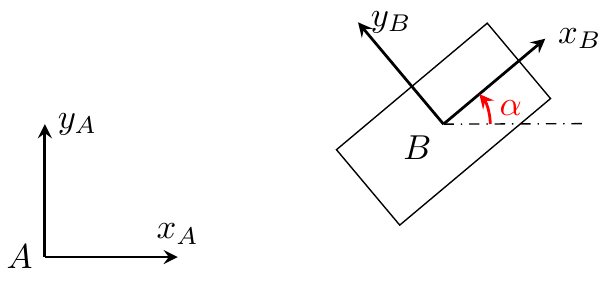}
  \caption{The world frame $A$ and body frame $B$ for a robot moving in the plane.}
  \label{fig_demoTwoFramesPlane}
\end{figure}

If a robot is moving in the plane, the rotation (or orientation) of the robot can be represented by a single angle $\alpha$ (see Figure~\ref{fig_demoTwoFramesPlane}).
Then, the rotation transformation from frame $B$ to frame $A$ is
\begin{align*}
R_\alpha=\left[
           \begin{array}{cc}
             \cos\alpha & -\sin\alpha \\
             \sin\alpha & \cos\alpha \\
           \end{array}
         \right].
\end{align*}
In order to verify $R_\alpha$, one may examine some specific points in frame $B$ such as $e_1=[1,0]^\T $ and $e_2=[0,1]^\T $.
Taking time derivative on both sides of $R_\alpha$ gives
\begin{align}\label{eq_RalphaTimeDot}
\dot{R}_\alpha=\left[
           \begin{array}{cc}
             -\sin\alpha & -\cos\alpha \\
             \cos\alpha & -\sin\alpha \\
           \end{array}
         \right]\dot{\alpha}.
\end{align}
Expression \eqref{eq_RalphaTimeDot} may also be obtained as a special case of \eqref{eq_rotationTimeDerivative2}.
In particular, we can consider the three-dimensional frame with the $z$-axes pointing out of the paper in Figure~\ref{fig_demoTwoFramesPlane}.
Then,
\begin{align*}
R_B^A=\left[
  \begin{array}{ccc}
    \cos\alpha & -\sin\alpha & 0 \\
    \sin\alpha & \cos\alpha & 0 \\
    0 & 0 & 1 \\
  \end{array}
\right].
\end{align*}
The angular velocity of the robot can be expressed as $w_A=w_B=\dot{\alpha} e_3$ where $e_3=[0,0,1]^\T $.
By applying \eqref{eq_rotationTimeDerivative1} or \eqref{eq_rotationTimeDerivative2}, it is straightforward to obtain \eqref{eq_RalphaTimeDot}.
\bibliography{myOwnPub,zsyReferenceAll} 
\bibliographystyle{ieeetr}

\end{document}